%% file: example.tex
\documentclass[submission,copyright,creativecommons]{eptcs}
\providecommand{\event}{ICLP 2025} %

\usepackage{iftex}

\ifpdf
  \usepackage{underscore}         %
  \usepackage[T1]{fontenc}        %
\else
  \usepackage{breakurl}           %
\fi

\title{An ASP-Based Framework for MUSes}
\author{Mohimenul Kabir
\institute{National University of Singapore}
\and
Kuldeep S Meel
\institute{Georgia Institute of Technology}
}
\def\titlerunning{An ASP-Based Framework for MUSes}
\def\authorrunning{Kabir and Meel}
\input{section/macro}
\begin{document}
\maketitle
\begin{abstract}
  Given an unsatisfiable formula, understanding the core reason for unsatisfiability is crucial in several applications. 
  One effective way to capture this is through the minimal unsatisfiable subset (MUS), the subset-minimal set of clauses that remains unsatisfiable.
  Current research broadly focuses on two directions: (i) enumerating as many MUSes as possible within a given time limit, and (ii) counting the total number of MUSes for a given unsatisfiable formula.
  
  In this paper, we introduce an answer set programming-based framework, named \toolname, designed for {\em online} enumeration of MUSes. 
  ASP is a powerful tool for its strengths in knowledge representation and is particularly suitable for specifying complex combinatorial problems.
  By translating MUS enumeration into answer set solving, \toolname~leverages the computational efficiency of state-of-the-art ASP systems.
  Our extensive experimental evaluation demonstrates the effectiveness of \toolname~and highlights the acceleration in both MUS enumeration and counting tasks, particularly when integrated within hybrid solvers, including the framework proposed in this paper.
  
  \noindent The extended version of this paper is available at: \url{https://arxiv.org/abs/2507.03929}. 
\end{abstract}

\input{section/introduction}
\input{section/preliminaries}
\input{section/related_work}

\input{section/algorithm}
\input{section/experiment}

\input{section/conclusion}
\nocite{*}
\bibliographystyle{eptcs}
\bibliography{example}
\end{document}

%% file: section/macro.tex
\usepackage{amsmath}
\usepackage{amsthm}
\usepackage{cleveref}
\usepackage{flexisym}
\usepackage{enumitem}
\usepackage{amssymb} %
\usepackage{array} %
\usepackage{color}
\usepackage{graphicx}
\usepackage{booktabs}
\usepackage{listings}
\usepackage{xspace}
\usepackage{multicol}
\usepackage{float}
\usepackage{booktabs}
\usepackage{array}
\usepackage{url}
\usepackage{subcaption}
\usepackage{caption}
\usepackage{amsfonts}
\usepackage{thm-restate}
\usepackage[mathcal]{eucal}
\usepackage{xcolor}
\definecolor{codegreen}{rgb}{0,0.6,0}
\definecolor{codegray}{rgb}{0.5,0.5,0.5}
\definecolor{codepurple}{rgb}{0.58,0,0.82}
\definecolor{backcolour}{rgb}{0.95,0.95,0.92}
\newtheorem{proposition}{Proposition}

\newtheorem{example}{Example}

\usepackage{etoolbox}
\makeatletter
\patchcmd{\@makecaption}
  {\scshape}
  {}
  {}
  {}
\makeatletter
\patchcmd{\@makecaption}
  {\\}
  {.\ }
  {}
  {}
\makeatother

\newcommand{\unsat}{\ensuremath{\mathsf{unsat}}}
\newcommand{\pos}[1]{\ensuremath{\mathsf{p}_{#1}}}
\newcommand{\ngt}[1]{\ensuremath{\mathsf{n}_{#1}}}
\newcommand{\cls}[1]{\ensuremath{\lambda_{#1}}}

\newcommand{\program}[1]{\mathcal{P}_{#1}}
\newcommand{\body}[1]{\mathsf{Body}(#1)}
\newcommand{\true}{\ensuremath{\mathsf{true}}\xspace}
\newcommand{\false}{\ensuremath{\mathsf{false}}\xspace}

\newcommand{\Card}[1]{|#1|}
\newcommand{\Var}[1]{\mathsf{Var}(#1)}
\newcommand{\clarknot}{\ensuremath{\mathsf{not}}\text{ }}
\newcommand{\answer}[1]{\mathsf{AS}(#1)}
\newcommand{\toolname}{\ensuremath{\mathsf{MUS}}-\ensuremath{\mathsf{ASP}}}
\newcommand{\marco}{MARCO\xspace}
\newcommand{\tome}{TOME\xspace}
\newcommand{\remus}{ReMUS\xspace}
\newcommand{\unimus}{UNIMUS\xspace}
\newcommand{\countmust}{CountMUST\xspace}
\newcommand{\amusic}{AMUSIC\xspace}

\newcommand{\be}{\mathbb{B}}
\newcommand{\heuristic}[1]{$\mathsf{\mathbf{H}}_{#1}$}

\newcommand{\rules}[1]{#1}
\newcommand{\at}[1]{\mathsf{atoms}(#1)}
\newcommand{\head}[1]{\mathsf{Head}(#1)}
\newcommand{\lb}[1]{\mathsf{lb}(#1)}

\newcommand{\qesshort}[1]{TQP\xspace}

\renewcommand{\lstlistingname}{Encoding}
\newcommand{\listingname}{Encoding}
\newcommand{\collect}{\ensuremath{\mathfrak{C}}}
\newcommand{\allclause}{\Lambda}
\newcommand{\ncls}{\ensuremath{\mathsf{ncl}}}

%% file: section/introduction.tex
\section{Introduction}
When dealing with an unsatisfiable formula --- one that has no solution --- it is often crucial to identify the reasons for its unsatisfiability.
A common concept in the literature is the {\em minimal unsatisfiable subset} (MUS), which provides a concise explanation for the formula's unsatisfiability~\cite{LPMM2016}.
For instance, in {\em requirement analysis}~\cite{Bendik2017}, a set of constraints represents the requirements that need to be fulfilled.
If these constraints are unsatisfiable, a MUS highlights the minimal set of requirements that cannot be satisfied simultaneously.

Given an unsatisfiable Boolean formula $F$, expressed as a conjunction of clauses $\{C_1, \ldots, C_{\ncls}\}$,
a minimal unsatisfiable subset (MUS) of $F$ is a subset $F\textprime \subseteq F$ such that $F\textprime$ is unsatisfiable
and for every $C_i \in F\textprime$, the set $F\textprime \setminus \{C_i\}$ is satisfiable.
Finding MUSes is an active research area due to their wide range of applications, including  
formal equivalence checking~\cite{CGLNR2010}, 
Boolean function bi-decomposition~\cite{CM2012},
counter-example guided abstraction refinement~\cite{ALS2007}, circuit error diagnosis~\cite{HL1999},
and others~\cite{HK2008,IMMV2016,JS2016}.

Since the MUS of an unsatisfiable formula is not necessarily unique, it is important to enumerate multiple MUSes. 
Enumerating more MUSes of $F$ provides a clearer explanation of the underlying unsatisfiability~\cite{BC2020}.
Thus, existing research focuses on enumerating MUSes in an {\em online fashion}~\cite{BK2015,BC2020,BCB2018,BK2016,LM2013,NBMS2018}, aiming to identify as many MUSes as possible within a given time limit. 
Moreover, the MUS count is frequently used as a metric to quantify inconsistency in diagnostic tasks~\cite{HK2008,MLJ2011}.    

Existing MUS enumeration algorithms can be broadly divided into two categories: algorithms based on the well-known {\em seed-shrink} scheme~\cite{BC2020b} and those employing the {\em hitting set duality}~\cite{Reiter1987}.  
The seed-shrink scheme systematically explores subsets of $F$ and identifies an unsatisfiable subset called a {\em u-seed}, which is then shrunk into a MUS through a single MUS {\em extractor}~\cite{BM2012}. 
Conversely, the hitting set duality leverages the well-known relationship between MUSes and MCSes, where a {\em minimal correction subset} (MCS) is the minimal subset $F\textprime \subset F$ such that the formula $F \setminus F\textprime$ is satisfiable. 
This approach involves computing and exploiting the MCSes of $F$ to enumerate MUSes~\cite{LS2008}.
Traditional MUS enumeration techniques often rely on SAT-based methods. 
In contrast, we explore an alternative perspective by formulating MUS enumeration as an {\em answer set solving} problem, leveraging the expressiveness of ASP for modeling combinatorial structures.

Answer set programming (ASP)~\cite{MT1999} is a well-known declarative problem-solving paradigm in knowledge representation and reasoning, valued for its ability to model complex combinatorial problems~\cite{EGL2016}. 
An ASP {\em program} consists of \textit{logical rules} defined over {\em propositional atoms}, encoding domain knowledge and queries. 
An assignment of values to propositional atoms that satisfies the ASP semantics of the program is referred to as an \emph{answer set}.
ASP programs can be classified into two categories based on their expressiveness: {\em disjunctive} and {\em normal} ASP programs~\cite{Lierler2005}. 
However, the expressiveness comes at the cost of computational complexity. 
Specifically, determining whether a disjunctive program has an answer set is $\textstyle \sum_{2}^{P}$-complete~\cite{EG1995}, 
while the complexity reduces to NP-complete for normal programs~\cite{MT1991}.

The main contribution of the paper is the proposal of an ASP solving-based MUS enumeration framework called \toolname.
At a high level, given an unsatisfiable Boolean formula $F$, we formulate a disjunctive ASP program $\program{F}$ such that the answer sets of $\program{F}$ correspond to the unsatisfiable cores of $F$.
To compute MUSes, we focus on {\em subset minimal} answer sets with respect to a {\em target} set of atoms.
Computing the subset minimal answer sets is not a new research problem~\cite{ADFPR2023,BDRS2015,GKROS2013}; it has been shown to be useful in addressing various real-world computational problems~\cite{ADFR2021,DGRRS2019,OPSSG2016}.
We are not the first to leverage ASP for MUS computation; it has previously been used to compute MUSes of LTL$_f$~\cite{IMRP2025}. 
Particularly, we harness the efficiency of subset minimal answer set solvers to compute MUSes. 
Moreover, we incorporate several heuristics into our encoding to enhance practical scalability.
Although the worst-case complexity of disjunctive ASP is high ($\textstyle \sum_{2}^{P}$-complete), practical instances often admit a {\em small backdoor} to {\em normality}~\cite{FS2015}, meaning that they 
can be transformed into normal programs by few atoms.
From the computational complexity perspective, finding a MUS is in $\text{FP}^{\text{NP}}$~\cite{CT1995,MJ2014}. 
It is worth noting that reducing one problem to a harder one is not uncommon in combinatorial problem-solving~\cite{LSK2023}. 
For example, the MUS counting framework AMUSIC~\cite{BM2020} invoked a $\textstyle \sum_{3}^{P}$ oracle to find MUSes.

Building on prior work in ASP-based counting~\cite{KM2023,KTPM2025}, we provide an extensive experimental evaluation demonstrating that \toolname~outperforms existing MUS enumerators and counters across multiple benchmarks, where counting is performed via enumeration. 
Although \toolname~does not consistently dominate across every benchmark, our proposed {\em hybrid} approach, which integrates \toolname~with existing enumerators, achieves significant performance gains across all benchmarks. 
While the {\em complete} MUS enumeration is theoretically intractable~\cite{Sperner1928}, \toolname~significantly outperforms state-of-the-art MUS counters in practice.

The paper is organized as follows: 
\Cref{section:preliminaries} presents the necessary background knowledge,
\Cref{section:relatedwork} reviews the existing literature on MUS enumeration and counting, 
\Cref{section:methodology} presents our proposed framework \toolname, designed for MUS enumeration,
\Cref{section:experiment} presents the experimental results,
finally, we conclude the paper in~\Cref{section:conclusion} with some future research directions. 

%% file: section/preliminaries.tex
\section{Preliminaries}
\label{section:preliminaries}
In this section, we present the background knowledge from propositional logic, hitting set duality, and answer set programming to understand the technical contribution. 

\paragraph{Propositional Logic} A \emph{propositional variable} $v$ takes a value from the domain $\be = \{0,1\}$, equivalently $\{\mathsf{false}, \mathsf{true}\}$. 
A \emph{literal} $\ell$ is either a variable $v$ (positive literal) or its negation $\neg{v}$ (negative literal).
A \emph{clause} $C$ is a {\em disjunction} of literals, denoted as $C = \bigvee_{i} \ell_i$. 
A Boolean formula $F$, in \emph{Conjunctive Normal Form (CNF)}, is a {\em conjunction}
of clauses, represented as $F = \bigwedge_{j} C_j$. We use the notation $\Var{F}$ to denote the set of propositional atoms or variables within  
$F$. 
An assignment $\tau$ is a mapping $\tau: X \rightarrow \be$, where $X \subseteq \Var{F}$.  For a variable $v \in X$, we define
$\tau(\neg{v}) = 1 - \tau(v)$. 
An assignment $\tau$ satisfies a literal $\ell$, denoted as $\tau \models \ell$, if $\tau$ evaluates $\ell$ to \true.
Similarly, an assignment $\tau$ satisfies a clause $C$, denoted as $\tau \models C$, if $\tau$ evaluates one of its literals to \true.
The assignment $\tau$ over $\Var{F}$ is a {\em model} of $F$ if $\tau$ evaluates $F$ to be \true. 
A formula $F$ is said to be {\em satisfiable} if there exists a model $\tau$ of $F$.
Otherwise, the formula $F$ is said to be {\em unsatisfiable}.

We consider a model of $F$ as a set of positive literals it assigns and define the {\em minimal model} of a propositional formula $F$~\cite{ABFL2017}.
We also consider an {\em ordering operator} over two models~\cite{JGM2010}.
For two given models $\tau_1$ and $\tau_2$, $\tau_1$ is considered {\em smaller} than $\tau_2$, denoted as $\tau_1 \leq \tau_2$, if and only if for each $x \in \Var{F}$, $\tau_1(x) \leq \tau_2(x)$; 
and $\tau_1$ is {\em strictly smaller} than $\tau_2$, denoted as $\tau_1 < \tau_2$, if $\tau_1 \leq \tau_2$ and $\tau_1 \neq \tau_2$.
A model $\tau$ is a {\em minimal model} of $F$ if and only if $\tau$ is a model of $F$ and no model of $F$ is strictly smaller than $\tau$.

It is well known that in every minimal model $\tau$, each variable assigned to \true is {\em justified}~\cite{KM2024}. More formally,  
given that $\tau$ is a minimal model of $F$, a (positive) literal $\ell \in \tau$ is said to be {\em justified} in $\tau$, if there exists a clause $C \in F$ such that $\ell \in C$ and $\tau \setminus \{\ell\} \not \models C$. 
That is, every literal $\ell \in \tau$ must be justified; otherwise, if no such $C$ exists for a particular $\ell$, then $\tau \setminus \{\ell\}$ is also a model of $F$ and $\tau \setminus \{\ell\} < \tau$, contradicting that $\tau$ is minimal. 
Intuitively, justification explains {\em why} each literal $\ell \in \tau$ must be in the model; removing it would violate the satisfaction of some clauses in $F$.

We consider a CNF formula as a multiset of clauses and use the operations introduced in the context of sets.
Within the context of unsatisfiable Boolean formula, we introduce the following terminologies. 
An unsatisfiable core of $F$ is a set of clauses $U \subseteq F$ such that $U$ is unsatisfiable. 
A {\em minimal correction subset} (MCS) of $F$ is a subset of clauses $F\textprime \subset F$ such that $F \setminus F\textprime$ is satisfiable and for every clause $C_i \in F\textprime$, it holds that $F \setminus (F\textprime \setminus \{C_i\})$ is unsatisfiable~\cite{LPMM2016}.
A {\em minimal unsatisfiable subset} (MUS) of $F$ is a subset of clauses $F\textprime \subseteq F$ such that $F\textprime$ is unsatisfiable and for every $C_i \in F\textprime$, it holds that $F\textprime \setminus \{C_i\}$ is satisfiable.

\begin{example}
    \label{example:example1}
    Consider the Boolean formula $F = \{C_1 = \{x_1\}, C_2 = \{\neg{x_1}\}, C_3 = \{x_2\}, C_4 = \{\neg{x_1}, \neg{x_2}\}\}$.
    The formula $F$ is unsatisfiable. 
    The formula $F$ has $5$ unsatisfiable cores: $\{C_1,C_2\}$, $\{C_1,C_2,C_3\}$, $\{C_1,C_2,C_4\}$, $\{C_1,C_3,C_4\}$, $\{C_1,C_2,C_3,C_4\}$. 
    The formula $F$ has $3$ MCSes: $\{C_1\}$, $\{C_2,C_3\}$, $\{C_2,C_4\}$.
    The formula $F$ has $2$ MUSes: $\{C_1,C_2\}$ and $\{C_1,C_3,C_4\}$.
\end{example}

\paragraph{Minimal Hitting Set Duality}
Given a collection of sets $\collect$, a hitting set $H$ is a set such that for every $\collect\textprime \in \collect$, $\collect\textprime \cap H \neq \emptyset$.
A hitting set is called minimal if none of its proper subsets is a hitting set. 
For a given Boolean formula $F$, the {\em minimal hitting set duality}~\cite{LMPS2005,Reiter1987} says that $H$ is a MUS of $F$ if and only if $H$ is a minimal hitting set of all MCSes of $F$.

\paragraph{Answer Set Programming}
An \textit{answer set program} $P$ consists of a set of rules, each rule $r$ is structured as either of the two following forms:
\begin{align}
\label{eq:generalrule}
\text{$r$:~}a_1 \vee \ldots \vee a_k \leftarrow b_1, \ldots, b_m, \textsf{not } c_1, \ldots, \textsf{not } c_n 
\end{align}
\begin{align}
    \label{eq:choicerule}
    \text{$r$:~}\{a_1 ; \ldots ; a_k\} \geq \lb{r}
    \leftarrow b_1, \ldots, b_m, \textsf{not } c_1, \ldots, \textsf{not } c_n
\end{align}
where, $a_1, \ldots, a_k, b_1, \ldots, b_m, c_1, \ldots, c_n$ are propositional variables or atoms; $k,m,n,\lb{r}$ are non-negative integers. 
The notations $\rules{P}$ and $\at{P}$ denote the rules and atoms within the program $P$, respectively. 
In rule $r$, the operator ``\textsf{not}'' denotes \textit{default negation}~\cite{clark1978}. For each 
rule $r$ (shown in \Cref{eq:generalrule,eq:choicerule}), we adopt the following notations: the atom set $\{a_1, \ldots, a_k\}$ constitutes the {\em head} of $r$, denoted by $\head{r}$, the set $\{b_1, \ldots, b_m\}$ is referred to as the {\em positive body atoms} of $r$, denoted by $\body{r}^+$, and the set $\{c_1, \ldots, c_n\}$ is referred to as the \textit{negative body atoms} of $r$, denoted by $\body{r}^-$.
A rule $r$ is called a {\em constraint} when $\head{r}$ has no atom. The {\em empty head} and {\em empty body} are often denoted by the notations $\bot$ (\false) and $\top$ (\true), respectively.
A rule $r$ is called a {\em cardinality} rule when the rule $r$ is of form~\Cref{eq:choicerule}.
A program $P$ is called a {\em disjunctive logic program} if there is a rule $r\in \rules{P}$ such that $r$ is of form~\cref{eq:generalrule} and $\Card{\head{r}} \geq 2$~\cite{BD1994}.
Otherwise, the program $P$ is called a normal logic program.
In this paper, we use the ASP-Core-$2$ standard input language for ASP programs~\cite{CFGIKK2020}.

In ASP, an interpretation $\tau$ over $\at{P}$ specifies which atoms are assigned \true; an atom $a$ is \true under $\tau$ if and only if $a \in \tau$ (or \false when $a \not\in \tau$ resp.). 
An interpretation $\tau$ satisfies a rule $r$, denoted by $\tau \models r$, if and only if: (i) $(\head{r} \cup \body{r}^{-}) \cap \tau \neq \emptyset$ or $\body{r}^{+} \setminus \tau \neq \emptyset$ for a rule of the form~\cref{eq:generalrule}; and (ii) $\Card{\head{r} \cap \tau} \geq \lb{r}$ or $\tau \cap \body{r}^{-} \neq \emptyset$ or $\body{r}^{+} \setminus \tau \neq \emptyset$ for a rule of the form~\cref{eq:choicerule}. 
An interpretation $\tau$ is a {\em model} of $P$, denoted by $\tau \models P$, if $\tau \models r$ for each $r \in \rules{P}$. 
The \textit{reduct} of a rule $r \in \rules{P}$, with respect to an interpretation $\tau$, denoted $r^{\tau}$, is defined as (i) the singleton $\{\head{r} \leftarrow \body{r}^+ \mid \body{r}^- \cap \tau = \emptyset\}$ if $r$ is of the form~\cref{eq:generalrule}
and (ii) the set $\{a \leftarrow \body{r}^+ \mid a \in \head{r} \cap \tau, \body{r}^- \cap \tau = \emptyset\}$ if $r$ is of the form~\cref{eq:choicerule}.
The \textit{Gelfond-Lifschitz (GL) reduct} of a program $P$ with respect to an interpretation $\tau$, is defined as $P^{\tau} = \{r^{\tau} \mid r \in \rules{P}\}$~\cite{GL1991}.
The interpretation $M$ is an {\em answer set} of $P$ if $M \models P$ and no $M\textprime \subset M$ exists such that $M\textprime \models P^{M}$.
We denote the answer sets of program $P$ using the notation $\answer{P}$.
We define the {\em subset minimal answer set} of a program $P$ with respect to a set of atoms $O \subseteq \at{P}$.
An answer set $M \in \answer{P}$ is a subset minimal answer set of $P$ if there is no other answer set $M\textprime \in \answer{P}$ such that $M\textprime \cap O \subset M \cap O$~\cite{ADFPR2023}.

%% file: section/related_work.tex
\section{Related Works}
\label{section:relatedwork}
Previous research on MUSes primarily focused on the extraction of a single MUS or computing one single MUS~\cite{BM2011a,BM2012,Nadel2010,NRS2014}.
However, with the increasing demand for computing multiple MUSes, the area of MUS enumeration has emerged as a significant focus of recent research~\cite{BC2020}. 
Very recently, there has been growing interest in MUS counting.

\paragraph{MUS Enumeration.}
FLINT~\cite{NBMS2018} utilizes a framework similar to those in~\cite{MP2008} for MUS enumeration, 
iteratively computing MUSes and MCSes of an unsatisfiable formula across multiple {\em rounds}, alternatively {\em relaxing} an unsatisfiable formula and {\em strengthening} a satisfiable one.

\marco~\cite{LPMM2016}, TOME~\cite{BBCB2016}, and \remus~\cite{BCB2018} are algorithms based on {\em seed-shrink} techniques.
To find a MUS of a Boolean formula $F$, these algorithms gradually explore subsets of $F$ to identify an unexplored subset that is unsatisfiable, called a {\em u-seed}.
This u-seed is then shrunk to a MUS through a single MUS extraction subroutine.
The main difference between these algorithms lies in how they identify the u-seed $S$.
\marco~iteratively identifies maximal unexplored subsets of $F$ and checks their satisfiability until a u-seed is found.
\tome~computes a chain of unexplored subsets of $F$ --- namely, $F_1 \subset \ldots F_p$, and then performs a {\em binary search} over the chain to identify MUSes.
\remus~employs a different technique to identify smaller u-seeds, which are easier to shrink.
In particular, \remus~finds the initial u-seed $S$ among the maximal unexplored subsets of $F$ and shrinks it into a MUS $S\textprime$.
To find subsequent MUSes, \remus~selects a set $R$ such that $S\textprime \subset R \subset S$ and recursively searches for u-seeds within the maximal unexplored subsets of $R$, 
decreasing the size of the u-seed with each recursive call.

The hitting set duality is specifically used in MUS enumerators such as CAMUS~\cite{LS2008} and DAA~\cite{BS2005}.
CAMUS operates in two phases: in the first phase, it computes MCSes of a formula using an {\em incremental} SAT solver; in the second phase, it enumerates MUSes by finding minimal hitting sets of the computed MCSes through a {\em recursive branching algorithm}.
DAA is based on the {\em Dualize and Advance} technique~\cite{GKMS2003} and, unlike CAMUS, incrementally computes both MCSes and MUSes without enumerating all MCSes.
A major drawback of these techniques is that the number of MCSes can grow exponentially.

The hitting set duality is also leveraged in the MUS enumerator MCS-MUS~\cite{BK2015}, which targets {\em constrained MCSes} instead of arbitrary ones. 
This approach incorporates several optimization techniques, such as {\em recursive model rotation}~\cite{BM2011a} and {\em clause set refinement}~\cite{NRS2014}.
In a follow-up work, Bacchus and Katsirelos~\cite{BK2016} extended MCS-MUS with a version called MCS-MUS-BT, which performs a backtracking search over a tree, exploiting information gathered from previous MUS enumerations to accelerate the finding of new MUSes. 

\paragraph{MUS Counting.}
MUS counting without explicit enumeration was first introduced by~\cite{BM2020}. They developed a {\em hashing-based} model counting framework, \amusic, which provides an $(\varepsilon, \delta)$-approximation~\cite{CMV2013}. 
The counting framework \amusic does not scale due to its reliance on $\sum_{3}^{P}$ oracles.
Later, Bendik and Meel~\cite{BM2021} proposed \countmust, a framework for exact MUS counting that leverages {\em subtractive reduction}~\cite{DHK2005} and {\em projected model counting}. 

\paragraph{Subset-minimal Answer Sets.}
A variety of techniques exist for computing subset-minimal answer sets.
The tool asprin~\cite{BDRS2015} supports subset minimality as a special case of user-defined {\em preferences}.
Clingo computes the subset-minimal answer sets using an approach similar to the OPTSAT algorithm~\cite{DGM2010}.
More recently, Alviano et al.~\cite{ADFPR2023} proposed a family of methods for computing subset-minimal answer sets within the ASP solver Wasp.

%% file: section/algorithm.tex
\section{Methodology}
\label{section:methodology}
In this section, we present the \toolname~framework for MUS computation. 
We first introduce an ASP program $\program{F}$ that captures the unsatisfiable cores of a given Boolean formula $F$, with each answer set corresponding to one such unsatisfiable core (\Cref{subsection:aspencoding}).
Then, we show the correctness of our methodology (\Cref{subsec:theoreticalcorrectness}).
Finally, we describe some heuristics that improve performance by reducing the search space while preserving completeness (\Cref{subsection:heuristic}).  
\subsection{ASP Encoding of \toolname}
\label{subsection:aspencoding}
Since MUSes correspond to subset-minimal unsatisfiable subsets, and ASP excels at expressing subset minimality constraints, we encode the MUS enumeration problem as an ASP program.
For a given Boolean formula $F$, the ASP program $\program{F}$ defines the following atoms:
\begin{itemize}
  \item $\cls{i}$: for each clause $C_i \in F$, we introduce a {\em clause selector atom} $\cls{i}$.
  Each interpretation $\tau$ over the atoms of $\program{F}$ {\em selects} a subset of clauses from $F$; a clause $C_i$ is selected by an interpretation $\tau$ if and only if $\cls{i} \in \tau$. 
  \item $\pos{x}$, $\ngt{x}$: for each variable $x \in \Var{F}$, we introduce two atoms: $\pos{x}$ and $\ngt{x}$. 
  
  These atoms follow the symbolic interpretations: $\pos{x}$ (resp. $\ngt{x}$) indicates that the variable $x$ is assigned \true (resp. \false). While 
  a variable $x$ cannot be assigned both values simultaneously, it is possible that for a given interpretation $\tau$, both $\pos{x}$ and $\ngt{x}$ belong to $\tau$.  
  \item $\unsat$: we introduce an atom $\unsat$ to indicate that at least one of the clauses selected by an interpretation is not satisfied according to propositional semantics.
\end{itemize}
\begin{figure*}
\begin{lstlisting}[caption={Program $\program{F}$.},label={code:astouc},captionpos=b,mathescape=true,escapechar=|,xleftmargin=0.5cm]
  $\pos{x_i} \vee \ngt{x_i} \leftarrow \top.$|\label{line:eachvariable}|
  $\unsat \leftarrow \cls{i}, \ngt{x_1}, \ldots \ngt{x_k}, \pos{x_{k+1}}, \ldots \pos{x_{k+m}}$.|\label{line:reach1}|
  $\bot \leftarrow$ not $\unsat$.|\label{line:unsat}|
  $\pos{x_i} \leftarrow \unsat$.|\label{line:posvalue}|
  $\ngt{x_i} \leftarrow \unsat$.|\label{line:negvalue}|
  $\{ \cls{1}; \ldots \cls{\ncls} \} \geq 0 \leftarrow \top.$|\label{line:allclauses}|
\end{lstlisting}
\renewcommand{\lstlistingname}{Algorithm}
\end{figure*}

We denote the set $\allclause = \{\cls{1}, \ldots, \cls{\ncls}\}$ and the {\em projection} of $\tau$ onto the atom set $\allclause$ as $\tau_{\downarrow \allclause}$. 
This projection extracts the selected clauses, allowing us to map each answer set to a corresponding unsatisfiable core.
Given $F = \bigwedge_{i=1}^{\ncls} C_i$, we introduce the ASP program $\program{F}$, where each unsatisfiable core of $F$ corresponds one-to-one with the answer sets of $\program{F}$. 
The ASP program $\program{F}$ is outlined in \listingname~\ref{code:astouc}. 

The core idea behind $\program{F}$ is to non-deterministically select a set of clauses from $F$, as specified by the rule in line~\ref{line:allclauses}.
Among the selected clauses, the program enforces that at least one clause is unsatisfied under any variable assignment, ensuring the subset forms an unsatisfiable core (enforced by the rule in line~\ref{line:unsat}). 
We formally prove that each answer set of $\program{F}$ corresponds bijectively to an unsatisfiable core of $F$ (see~\Cref{theorem:astoucproof}).
More specifically, for every answer set $M \in \answer{\program{F}}$, the projection $M_{\downarrow \allclause}$ identifies an unsatisfiable core of $F$.
For instance, if $M_{\downarrow \allclause} = \{\cls{1}, \ldots, \cls{k}\}$, then the corresponding clause set $\{C_1, \ldots, C_k\} \subseteq F$ is unsatisfiable.
Consequently, all subset minimal answer sets of $\program{F}$ with respect to the set $\allclause$ capture exactly the minimal unsatisfiable subsets (MUSes) of $F$.
\begin{restatable}{theorem}{astoucproof}
  \label{theorem:astoucproof}
  There is a bijective correspondence between the answer sets of $\program{F}$ and the unsatisfiable cores of $F$, under the projection onto $\allclause$.
\end{restatable}
\noindent The proof of~\Cref{theorem:astoucproof} is deferred to~\Cref{subsec:theoreticalcorrectness}. 

\begin{proposition}
The subset minimal answer sets of $\program{F}$ with respect to the atom set $\allclause$ correspond to the MUSes of $F$.
\end{proposition}

\addtocounter{example}{-1}
\begin{example}[continued]
  Given the Boolean formula $F = \{C_1 = \{x_1\}, C_2 = \{\neg{x_1}\}, C_3 = \{x_2\}, C_4 = \{\neg{x_1}, \neg{x_2}\}\}$,
  the program $\program{F}$ consists of the following rules:
  \begin{lstlisting}[label={code:astouc},captionpos=b,mathescape=true,escapechar=|,numbers=none]
    $\pos{x_1} \vee \ngt{x_1} \leftarrow \top. \qquad \pos{x_2} \vee \ngt{x_2} \leftarrow \top.$   
    $\unsat \leftarrow \cls{1}, \ngt{x_1}. \qquad \unsat \leftarrow \cls{2}, \pos{x_1}. \qquad \unsat \leftarrow \cls{3}, \ngt{x_2}. \qquad \unsat \leftarrow \cls{4}, \pos{x_1}, \pos{x_2}.$
    $\bot \leftarrow$ not $\unsat$.
    $\pos{x_1} \leftarrow \unsat. \qquad \ngt{x_1} \leftarrow \unsat. \qquad \pos{x_2} \leftarrow \unsat. \qquad \ngt{x_2} \leftarrow \unsat.$
    $\{ \cls{1}; \cls{2}; \cls{3}; \cls{4} \} \geq 0 \leftarrow \top.$
  \end{lstlisting}
  The program $\program{F}$ has $5$ answer sets, each corresponding to an unsatisfiable core of $F$. 
  Among them, two are subset minimal answer sets w.r.t. the atom set $\allclause$ : $\{\pos{x_1}, \ngt{x_1},\pos{x_2}, \ngt{x_2}, \unsat, \cls{1}, \cls{2}\}$ and 
  \\$\{\pos{x_1}, \ngt{x_1},\pos{x_2}, \ngt{x_2}, \unsat, \cls{1}, \cls{3}, \cls{4}\}$. 
  These two subset minimal answer sets correspond to the MUSes $\{C_1, C_2\}$ and $\{C_1, C_3, C_4\}$, respectively.
\end{example}

\subsection{Theoretical Correctness}
\label{subsec:theoreticalcorrectness}
We outline the reduct of program $\program{F}$ with respect to an arbitrary interpretation $\tau \models \program{F}$ in \listingname~\ref{code:programtoreduct}. 
In both encodings, the lines starting with the \% sign indicate comments.
The reduct of $\program{F}$ w.r.t. $\tau$ is important for the theoretical correctness of our technique (proof of~\Cref{theorem:astoucproof}). 
Following the approach of SAT-based ASP solvers~\cite{Lierler2005}, the reduct replaces each rule $r \in \rules{\program{F}}$ by an {\em implication}: $\bigwedge_{b \in \body{r}^{+}} b \wedge \bigwedge_{c \in \body{r}^{-}} \neg{c} \longrightarrow \bigvee_{a \in \head{r}} a$.
These implications are written as disjunctions of literals in Encoding~\ref{code:programtoreduct}.
For each interpretation $\tau \models \program{F}$, we can show that $\{\unsat\} \subset \tau$ (\Cref{prop:helperproposition}).
As a result, following the definition of reduct, there is no implication corresponding to the rule in line~\ref{line:unsat} (Encoding~\ref{code:astouc}). 

\begin{proposition}
  \label{prop:helperproposition}
  For each $\tau \models \program{F}$: (i)~$\{\unsat\} \subset \tau$ and (ii)~$\forall x \in \Var{F}, \{\pos{x}, \ngt{x}\} \subset \tau$.
\end{proposition}
\Cref{prop:helperproposition} can be proved using {\em Clark completion} semantics~\cite{clark1978}. 
If (i) is not true, then $\tau$ does not satisfy the rule in line~\ref{line:unsat} of \listingname~\ref{code:astouc}.
If (ii) is not true, then $\tau$ does not satisfy rules in lines~\ref{line:posvalue} and~\ref{line:negvalue}.

\astoucproof*
\begin{proof}
    The proof is structured into two parts:
    \begin{enumerate}[leftmargin=*, itemsep=0pt, topsep=2pt]
      \item For each $M \in \answer{\program{F}}$, there exists a corresponding unsatisfiable core of $F$
      \item For each unsatisfiable core $U$ of $F$, there exists a corresponding answer set of $\program{F}$ 
    \end{enumerate}
    \textbf{Proof of ``Part $1$''}: Given $M \in \answer{\program{F}}$, without loss of generality, assume that $M_{\downarrow \allclause} = \{\cls{1}, \ldots, \cls{k}\}$.
    We prove that the clause set $\{C_1, \ldots, C_k\}$ is an unsatisfiable core of $F$.
    Each answer set follows Clark completion semantics and according to~\Cref{prop:helperproposition}, the answer set is $M = \{\cls{1}, \ldots, \cls{k}\} \cup \{\unsat\} \cup \{\pos{x}, \ngt{x}| \forall x \in \Var{F}\}$.
    Note that for each answer set $M$, the set $M_{\downarrow \allclause}$ maps to a unique set of clauses of $F$.
  
    For the purpose of contradiction, assume that the clause set $\{C_1, \ldots, C_k\}$ is satisfiable. 
    Let the clause set $\{C_1, \ldots, C_k\}$ have the variable set $\{x_1, \ldots, x_n\}$.
    Since the clause set $\{C_1, \ldots, C_k\}$ is satisfiable, there exists a model, which assigns variables $x_1, \ldots, x_{n^{\star}}$ to \true and $x_{n^{\star} + 1}, \ldots, x_n$ to \false, where $n^{\star} \in [1, n]$.
    Now consider an interpretation $\tau\textprime = \{\cls{1}, \ldots, \cls{k}, \pos{x_1}, \ldots, \pos{x_{n^{\star}}}, \ngt{x_{n^{\star} + 1}}, \ldots, \ngt{x_n}\}$, which is a proper subset of $M$.
    There is a model of the clause set $\{C_1, \ldots, C_k\}$, meaning that the model satisfies each clause and we can easily show $\tau\textprime$ satisfies $\program{F}^{M}$, where $\program{F}^{M}$ denotes the reduct of $\program{F}$ with respect to $M$.
    Since $\tau\textprime \subset M$, this contradicts that $M$ is an answer set of $\program{F}$.
  
    \textbf{Proof of ``Part $2$''}: Without loss of generality, let assume that the set of clauses $U = \{C_{1}, \ldots C_{k}\}$ is an unsatisfiable core of $F$.
    We prove that the unsatisfiable core $U$ uniquely corresponds to an answer set of $\program{F}$. More specifically, we prove that given the unsatisfiable core $U$, 
    the interpretation $\tau = \{\cls{1}, \ldots, \cls{k}\} \cup \{\unsat\} \cup \{\pos{x}, \ngt{x}| \forall x \in \Var{F}\}$ is an answer set of $\program{F}$.
    Note that the interpretation $\tau$ is unique for each unsatisfiable core $U$. 
    By definition, $\tau \models \program{F}$. We prove that no $\tau\textprime \subset \tau$ exists such that $\tau\textprime \models \program{F}^{\tau}$.
    
    For the sake of contradiction, assume that $\exists \tau\textprime \subset \tau$ such that $\tau\textprime$ is a minimal interpretation satisfying $\program{F}^{\tau}$.
    Recall from justification definition, as explained above, each atom in the minimal interpretation $\tau\textprime$ is justified. Otherwise, there is another interpretation $\tau\textprime\textprime \subset \tau\textprime$ that satisfies $\program{F}^{\tau}$.
    
    Note that $\{\cls{1}, \ldots, \cls{k}\} \subset \tau\textprime$ (due to the clause introduced in line~\ref{line:unitclause} of Encoding~\ref{code:programtoreduct}).
    Furthermore, $\unsat \not \in \tau\textprime$; if $\unsat \in \tau\textprime$ (or $\unsat$ is justified), then $\tau\textprime = \tau$ (due to the clauses in lines~\ref{line:posinreduct},\ref{line:neginreduct} of Encoding~\ref{code:programtoreduct}).
    For each variable $x \in \Var{F}$, it holds that $\{\pos{x}, \ngt{x}\} \cap \tau\textprime \neq \emptyset$ (due to the clause in line~\ref{line:eachvariable} of Encoding~\ref{code:programtoreduct}).
    Since $\tau\textprime$ is a minimal interpretation satisfying $\program{F}^{\tau}$, $\forall x \in \Var{F}$, we can say $\{\pos{x}, \ngt{x}\} \not \subset \tau\textprime$; because only one atom from $\{\pos{x}, \ngt{x}\}$ can be justified from the clause in line~\ref{line:eachvariable} of Encoding~\ref{code:programtoreduct}.
    No further clause exists in Encoding~\ref{code:programtoreduct} to justify both $\pos{x}$ and $\ngt{x}$. 
    From the interpretation $\tau\textprime \models \program{F}^{\tau}$, we can derive an assignment over variables that satisfies each clause of $U$;
    otherwise, the atom $\unsat$ is justified in $\tau\textprime$ due to clause in line~\ref{line:reach1} of Encoding~\ref{code:programtoreduct} and it follows that $\tau = \tau\textprime$ (it contradicts our assumption). 
    The existence of a satisfying assignment of the clause set $U$ contradicts that $U$ is an unsatisfiable core.
\end{proof}
\begin{figure*}
  \begin{lstlisting}[caption={Reduct of $\program{F}$ with respect to any interpretation $\tau \models \program{F}$. The lines with \% are comments.},label={code:programtoreduct},captionpos=b,mathescape=true,escapechar=|,xleftmargin=0.5cm]
    $\pos{x_i} \vee \ngt{x_i}$|\label{line:eachvariable}|
    $\neg \cls{i} \vee \neg \ngt{x_1} \vee \ldots \vee \neg \ngt{x_k} \vee \neg \pos{x_{k+1}} \vee \ldots \vee \neg \pos{x_{k+m}} \vee \unsat$|\label{line:reach1}|
    $\neg \unsat \vee \pos{x_i}$|\label{line:posinreduct}|
    $\neg \unsat \vee \ngt{x_i}$|\label{line:neginreduct}|
    $\cls{i}$|\label{line:unitclause}|
  \end{lstlisting}
  \renewcommand{\lstlistingname}{Algorithm}
\end{figure*}

\subsection{Optimizing MUS Computation with Heuristics}
\label{subsection:heuristic}
Following the encoding $\program{F}$, we introduce additional constraints to accelerate MUS computation. 
These optimization techniques are based on observations derived from MUS counting~\cite{BM2021}.
We incorporate these observations into the ASP encoding of $\program{F}$ to enhance the efficiency of the~\toolname~framework. 

\paragraph{\heuristic{1}~Union of MUSes.}
If we compute a union of MUSes or an over-approximation of them, 
then we can focus only on those clauses for MUS computation~\cite{LL2024}. 
Computing the union of MUSes is computationally harder~\cite{MKIM2019}.
Building on the work in~\cite{BM2021}, we exploit a {\em lean kernel} based approach, by invoking a MaxSAT solver~\cite{KK2009}, to compute the over-approximation of all MUSes~\cite{MIMML2014}.   

\paragraph{\heuristic{2}~Minimum and Maximum Cardinalities of MUSes.}
If we compute the minimum (or a lower bound) and maximum (or an upper bound) cardinality of MUSes, then we can encode the cardinality information as an ASP cardinality rule~\cite{CFGIKK2020} (Line~\ref{line:allclauses} of Encoding~\ref{code:astouc}) to reduce the search space of $\program{F}$. 
Since finding the minimum cardinality MUS is intractable~\cite{IPLM2015},
we adopt the technique from~\cite{BM2021} to compute a lower bound on the minimum MUS cardinality using the well-known minimal hitting set duality.
Additionally, we compute an upper bound on the maximum MUS cardinality by solving a {\em maximum satisfiability} problem.

\paragraph{\heuristic{3}~Component Decompositions.}
Bendik and Meel~\cite{BM2021} partition the clauses into {\em disjoint components}, exploiting the fact that all clauses in a MUS originate from the same component. 
The computation of disjoint components involves constructing an undirected graph 
as follows: each clause $C_i$ of the given formula corresponds to a node $v_i$ and an edge exists between two nodes $v_j$ and $v_k$, if there is a literal $\ell$ such that $\ell \in C_j$ and $\neg{\ell} \in C_k$.
To reduce the search space based on the concept of disjoint components, \toolname~introduces constraints of the form: $\bot \leftarrow \cls{i}, \cls{j}$, where $C_i$ and $C_j$ are two clauses from different components in the graph constructed above. 
The constraint ensures that both atoms $\cls{i}$ and $\cls{j}$ cannot appear in the same interpretation. 

\paragraph{\heuristic{4}~Hitting Set Duality.}
Each MUS of $F$ must intersect every known MCS. 
This observation is also utilized by~\cite{BM2021}.
Following a similar approach, we compute MCSes of $F$ given a specific time limit. 
For each MCS represented by the clause set $\{C_{i_1}, \ldots, C_{i_k}\}$, we add the constraint: $\bot \leftarrow \clarknot \cls{i_1}, \ldots, \clarknot \cls{i_k}$, ensuring that at least one clause from each known MCS appears in every answer set.

\paragraph{\heuristic{5}~Negative Literal Cover.}
If a clause $C_i$ is in a MUS $F\textprime$ such that $\exists \ell \in C_i$, then there exists another clause $C_j \in F\textprime$ such that $\neg{\ell} \in C_j$.
We incorporate this observation into \toolname~as follows: let $C_i$ be a clause in $F$ such that $\ell \in C_i$ and let $C_{j_1}, \ldots, C_{j_k}$ be the clauses containing the literal $\neg{\ell}$,
we introduce a cardinality rule as: $\{\cls{j_1}; \ldots \cls{j_k}\} \geq 1 \leftarrow \cls{i}$, ensuring that at least one of the clauses containing $\neg{\ell}$ is selected when $C_i$ is selected. 

These heuristics constrain the search space of $\program{F}$ to exclude some unsatisfiable cores of $F$ that are not MUSes of $F$. 
Since these heuristics introduce some additional constraints to $\program{F}$ that prune non-MUSes without excluding any MUS, thereby preserving the completeness of MUS enumeration~\cite{KESHFM2022}. 
Together, the encoding and heuristics enable \toolname~to efficiently enumerate MUSes by pruning the search space while ensuring completeness.

%% file: section/experiment.tex
\section{Experimental Evaluation}
\label{section:experiment}
We evaluated \toolname~through extensive experiments on available MUS benchmarks.
Our performance evaluation focused on addressing the following research questions:
\begin{enumerate}[label=\textbf{RQ\arabic*},itemindent=*]
  \item \label{q:q1} How does \toolname~compare to existing MUS enumerators?
  \item \label{q:q2} Can \toolname~outperform dedicated MUS counters in exact MUS counting?
  \item \label{q:q3} How does a hybrid integration of \toolname~improve efficiency across benchmarks?
\end{enumerate}

\paragraph{Design Choice}
We prototyped \toolname~with Clingo~\cite{GKS2012}, as the underlying ASP solver\footnote{\label{sharedfootnote}The source code, dataset, and binaries are available at: \url{https://zenodo.org/records/15300472}}.
In experiment, we marked the $\allclause$ atoms (see~\Cref{subsection:aspencoding}) for minimization and 
invoked Clingo~with the command \texttt{--enum-mode=domRec --heuristic=domain} to ensure subset minimality over $\allclause$ atoms. 
Note that we have also experimented with Wasp~\cite{ADLR2015}, which does not outperform Clingo on our benchmark set.

\subsection{Benchmark, Baseline and Environmental Setup}
To analyze~\ref{q:q1}, we collected our benchmarks from three sources:~(i)~the MUS track of the SAT $2011$ competition\footnote{\url{http://www.cril.univ-artois.fr/SAT11/}}~(ii)~random graph coloring instances~\cite{LL2024}, and 
(iii)~randomly generated formulas following the procedure outlined in~\cite{SLBLdD2018}, with settings consistent with~\cite{LL2024}. 
We evaluated \toolname~against three off-the-shelf MUS enumeration tools: 
(i)~\marco~\cite{LPMM2016}~(ii)~\unimus~\cite{BC2020}~and~(iii)~\remus~\cite{BCB2018}.

To analyze~\ref{q:q2}, we considered the MUS counting benchmarks\textsuperscript{\ref{sharedfootnote}}~\cite{BM2020,BM2021}.
Clingo counts MUSes by enumeration.
We baselined \toolname~against the exact MUS counter \countmust~\cite{BM2021}.
We invoked \countmust~with the default parameter settings, which performed best in prior MUS counting evaluations~\cite{BM2021}.
We excluded \amusic~from our comparison because it provides only approximate counts and was significantly outperformed by other MUS counters in previous studies.
A direct comparison with existing ASP counters~\cite{KCM2024,KCM2025} or minimal model counters~\cite{Kabir2024} is not feasible, as MUSes correspond to subset-minimal answer sets with respect to a specific set of atoms.

We conducted our experiments on machines with AMD EPYC $7713$ processors. 
Each benchmark instance was run on $1$ core, $16$GB memory, and a timeout of $3600$ seconds.

\paragraph{Evaluation Metric.}
For~\ref{q:q1}, we evaluated the performance of these MUS enumerators based on the number of MUSes enumerated and their {\em average ranks}.
We assign each solver a {\em rank} based on the number of MUSes it enumerates per instance (the lower ranks are better), as in~\cite{BC2020}. 
For~\ref{q:q2}, we evaluated the MUS counters based on the number of instances solved and their PAR$2$ scores~\cite{SAT2017}.

\subsection{Analysis of \ref{q:q1}.}
To address~\ref{q:q1}, we evaluated \toolname~with other MUS enumerators using their average ranks. 
The statistics are presented in~\Cref{table:musenumerationresult}, 
which indicate that \toolname~outperforms the existing MUS enumerators in terms of average ranks.

We compared the number of MUSes enumerated by \toolname~against other baseline tools (the comparison is visually presented in the extended version\footnote{The extended version of this paper is available at: \url{https://arxiv.org/abs/2507.03929}.}).
Our experimental analysis demonstrates that \toolname~enumerates more MUSes than \marco, \remus, and \unimus~in $555$, $494$, and $508$ instances, respectively. 
Moreover, \toolname~enumerates at least twice as many MUSes compared to \marco, \remus, and \unimus~in $550$, $394$, and $446$ instances, respectively.
In contrast, there are $86, 154, 141$ instances where \marco, \remus, and \unimus~enumerates more MUSes than \toolname, and the limitation is addressed in our hybrid approach discussed in~\ref{q:q3}. 

\begin{table}[h]
  \centering
  \begin{tabular}{m{3em} m{4em} m{4em} m{4em} m{4em}} 
  \toprule
   & {\footnotesize \marco} & {\footnotesize \remus} & {\footnotesize \unimus} & {\footnotesize \toolname}\\
  \midrule
  Rank & 2.882 & 2.917 & 2.483 & \textbf{1.576}\\
  \bottomrule
\end{tabular}
\caption{Average ranks of \toolname~and other MUS enumerators; {\em the lower ranks are better}.}
\label{table:musenumerationresult}
\end{table}

\subsection{Analysis of \ref{q:q2}.}
To address~\ref{q:q2}, we compared the number of instances solved by \toolname~and \countmust~and their corresponding PAR$2$ scores. 
The statistics are shown in~\Cref{table:muscountingresult}.
We consider an instance solved by \toolname~if Clingo is able to enumerate all MUSes within the timeout.
Their runtime performance is visually shown in~\Cref{fig:counterruntime}. 
Our experimental analysis reveals that \toolname~solved all instances that were solved by \countmust~and $38$ more instances than \countmust.
\begin{table}[h]
  \centering
  \begin{tabular}{m{5em} m{4em} m{5em} m{5em}} 
  \toprule
   & {\footnotesize \unimus} & {\footnotesize \countmust} & {\footnotesize \toolname} \\
  \midrule
  \#Solved & 545 & 887 & \textbf{925}\\
  \midrule  
  PAR$2$ & 3630 & 1410 & \textbf{1150} \\
  \bottomrule
\end{tabular}
\caption{The performance comparison of \toolname~and existing MUS counters on MUS counting.}
\label{table:muscountingresult}
\end{table}
\begin{figure}
  \centering
  \scalebox{0.8}{
    \includegraphics[width=0.75\linewidth]{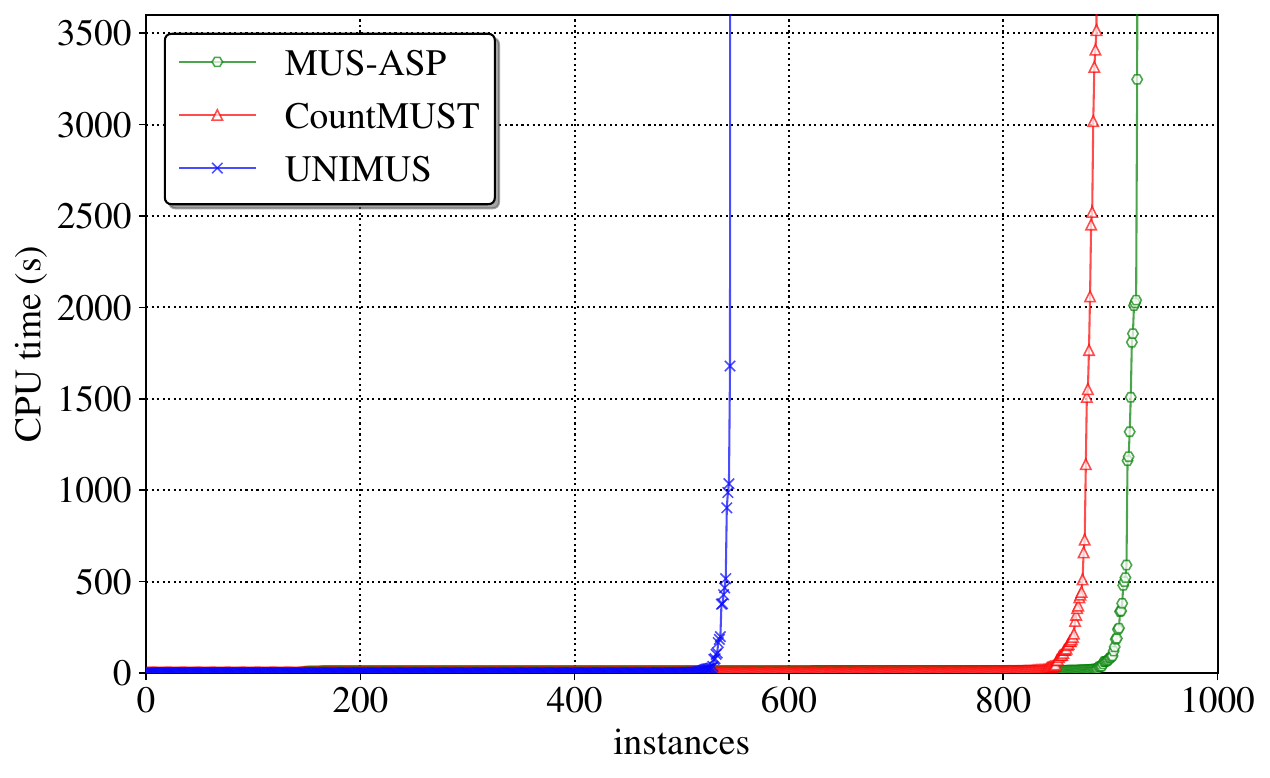}
  }
    \caption{Runtime performance of \toolname~and existing MUS counters on MUS counting. A point $(x,y)$ on the plot denotes that a tool counted $x$ instances within $y$ seconds.}
    \label{fig:counterruntime}
\end{figure}
\subsection{Analysis of \ref{q:q3}.}
We observe that the performance of \toolname~deteriorates with increasing formula size.
Specifically, there are $85, 92, 92$ instances where \toolname~failed to enumerate any MUS, while \marco, \remus, \unimus, respectively, succeeded in enumerating at least one MUS.
To address this limitation, we developed a hybrid enumerator, denoted as X+\toolname, for a given MUS enumerator X. 
The hybrid enumerator X+\toolname~invokes our prototype \toolname~for instances with fewer than $5000$ clauses (a threshold chosen empirically) and defaults to the enumerator X for larger instances.
We identified only two instances where X+\toolname~failed to enumerate any MUS, while the MUS enumerator X was able to enumerate at least one MUS.
This indicates that the hybrid integration almost universally inherits the better behavior of both \toolname~and the baseline X.
The performance of the hybrid enumerator is reported in~\Cref{table:musenumerationresultwithhybrid}.
The ranks demonstrate that integrating \toolname~with existing MUS enumerators substantially enhances overall performance.

\begin{table}[h]
    \centering
    \begin{tabular}{m{3em} m{9em} m{9em} m{9em} m{4em}} 
    \toprule
     & {\footnotesize \marco +\toolname} & {\footnotesize \remus +\toolname} & {\footnotesize \unimus +\toolname} & {\footnotesize \toolname}\\
    \midrule
    Rank & 1.231 & 1.110 & \textbf{1.053} & 1.405\\ 
    \bottomrule
    \end{tabular}
    \caption{Average ranks of hybrid MUS enumerators X+\toolname. 
    The hybrid approach uses \toolname~for instances having $<5000$ clauses; otherwise, it defaults to the MUS enumerator X.
    }
    \label{table:musenumerationresultwithhybrid}
\end{table}
\subsection{Further Experimental Analysis}
We evaluate the effectiveness of different heuristics incorporated within the \toolname~framework. 
For MUS enumeration, the results are more nuanced. We found that all five heuristics perform optimally, whereas excluding \heuristic{5} leads to significantly worse performance.
In contrast, MUS counting yields clearer results, which are presented in~\Cref{table:muscountingresultwithheuristic}.
In the table, the leftmost column (no heuristic) reports the baseline performance of \toolname~without any heuristic applied --- measured by the number of solved instances and PAR2 scores. 
The notation \heuristic{i..j} ($i < j$) indicates that the heuristics \heuristic{i},$\ldots$\heuristic{j} are enabled in~\toolname. 
The~\Cref{table:muscountingresultwithheuristic} confirms that each heuristic, even when applied individually, contributes to enhanced performance in the MUS counting task.

\begin{table}[h]
  \centering
  \begin{tabular}{m{5em} m{6em} m{4em} m{4em} m{4em}} 
  \toprule
   & no heuristic & \heuristic{1..2} & \heuristic{1..3} & \heuristic{1..5}\\
  \midrule
  \#Solved & 303 & 305 & 655 & \textbf{927}\\
  \midrule
  PAR$2$ & 5273 & 5261 & 2932 & \textbf{1135}\\
  \bottomrule
  \end{tabular}
  \caption{Performance of \toolname~on MUS counting benchmarks with different heuristic combinations.}
  \label{table:muscountingresultwithheuristic}
\end{table}
\noindent {\em Further experimental analysis is provided in the extended version.}

%% file: section/conclusion.tex
\section{Conclusion}
\label{section:conclusion}

This paper presented \toolname, an ASP-based framework for MUS enumeration that encodes the problem as the computation of subset-minimal answer sets of a disjunctive logic program. 
Experimental results demonstrate that \toolname~outperforms state-of-the-art MUS enumeration and counting tools across a diverse set of benchmarks, with additional gains from heuristic optimizations and hybrid integration. 
As a direction for future work, we aim to develop approximate MUS counting techniques by leveraging ASP solvers that support XOR constraints \cite{KESHFM2022}.